\documentclass[letterpaper, 10 pt, conference]{ieeeconf}  % Comment this line out if you need a4paper

\IEEEoverridecommandlockouts                              % This command is only needed if 
\input{irom.sty}

\usepackage{times}

\usepackage{multicol}

\usepackage{epsfig}
\usepackage{epstopdf}
\usepackage{graphicx}
\graphicspath{{Figures/}}
\usepackage{url}
\usepackage{stfloats}
\usepackage{amsmath}
\usepackage{amsfonts}   % to get mathbb for the reel symbol
\usepackage{amsopn}     % to get \DeclareMathOperator
\usepackage{psfrag}
\usepackage{mdwtab}     %ERW: added to get pretty tables
\usepackage{enumerate}
\usepackage{color}
\usepackage{accents}
\usepackage{mathtools}
\usepackage{algorithm}
\usepackage{algpseudocode}
\usepackage{adjustbox}
\usepackage{cases}
\usepackage[british]{babel}
\usepackage{hhline}
\usepackage{multirow}

\newtheorem{theorem}{\bf Theorem}

\title{\LARGE \bf
Stronger Generalization Guarantees for Robot Learning by \\ Combining Generative Models and Real-World Data
}

\author{Abhinav Agarwal, Sushant Veer, Allen Z. Ren, and Anirudha Majumdar% <-this % stops a space
\thanks{A. Agarwal, A. Z. Ren, and A. Majumdar are with the Department of Mechanical and Aerospace Engineering, Princeton University, Princeton, NJ, 08544.
S. Veer is with NVIDIA Research, Santa Clara, CA 95051, U.S.A. This work was conducted while S. Veer was with Princeton University.
        Emails: {\tt\small \{abhinav.agarwal, allen.ren, ani.majumdar\}@princeton.edu, sveer@nvidia.com}}%
\thanks{The authors were supported by the NSF CAREER award [2044149], the Office of Naval Research [N00014-21-1-2803], and the Toyota Research Institute (TRI). This article solely reflects the opinions and conclusions of its authors and not ONR, NSF, TRI or any other Toyota entity.}
}

\begin{document}

\maketitle
\thispagestyle{empty}
\pagestyle{empty}

%%%%%%%%%%%%%%%%%%%%%%%%%%%%%%%%%%%%%%%%%%%%%%%%%%%%%%%%%%%%%%%%%%%%%%%%%%%%%%%%
\begin{abstract}
We are motivated by the problem of learning policies for robotic systems with rich sensory inputs (e.g., vision) in a manner that allows us to guarantee generalization to environments unseen during training. We provide a framework for providing such \emph{generalization guarantees} by leveraging a finite dataset of real-world environments in combination with a (potentially inaccurate) generative model of environments. The key idea behind our approach is to utilize the generative model in order to \emph{implicitly} specify a \emph{prior} over policies. This prior is updated using the real-world dataset of environments by minimizing an upper bound on the expected cost across novel environments derived via \emph{Probably Approximately Correct (PAC)-Bayes} generalization theory. We demonstrate our approach on two simulated systems with nonlinear/hybrid dynamics and rich sensing modalities: (i) quadrotor navigation with an onboard vision sensor, and (ii) grasping objects using a depth sensor. Comparisons with prior work demonstrate the ability of our approach to obtain stronger generalization guarantees by utilizing generative models. We also present hardware experiments for validating our bounds for the grasping task. 
\end{abstract}

%%%%%%%%%%%%%%%%%%%%%%%%%%%%%%%%%%%%%%%%%%%%%%%%%%%%%%%%%%%%%%%%%%%%%%%%%%%%%%%%
\section{Introduction}

The ability of modern deep learning techniques to process high-dimensional sensory inputs (e.g., vision or depth) provides a promising avenue for training autonomous robotic systems such as drones, robotic manipulators, or autonomous vehicles to operate in complex and real-world environments. However, one of the fundamental challenges with current learning-based approaches for controlling robots is their limited ability to \emph{generalize} beyond the specific set of environments they are trained on \cite{Sunderhauf18}. This lack of generalization is a particularly pressing problem for safety- or performance-critical systems for which one would ideally like to provide \emph{formal guarantees} on generalization to novel environments.

% The datasets used to train these systems can have two sources: (i) real-world dataset, and (ii) hand-engineered generative model. Real-world datasets are, as the name suggests, realistic. However, they are a scarce resource and need to be carefully curated and/ or acquired, which can be painstaking and time-consuming. Generated datasets, on the other hand, can be a source of infinite data but they are usually not as accurate. For instance, consider an arm trying to grasp some mugs. The number of real mugs that can be modelled and used for training will be limited depending on capacity and space. On the other hand, easy to produce digitally generated shapes similar to mugs (like hollow cylinders), can be produced in large numbers, but won't very accurately represent the shape of the mugs we want the arm to grasp. \\

A primary contributing factor to this challenge is the fact that real-world datasets for training robotic systems are often limited in size (e.g., in comparison to large-scale datasets available for training visual recognition models via supervised learning). Such datasets often have to be carefully and painstakingly curated, e.g., by scanning indoor environments using 3D cameras for creating a dataset for visual navigation tasks \cite{armeni_cvpr16, replica19arxiv, xia2020interactive}, or scanning objects and characterizing their physical properties (e.g., inertia, friction, and mass) for creating a dataset for robotic manipulation~\cite{mahler2017dex, calli2017yale, chang2015shapenet}. One way to address this challenge of scarce real-world data is to leverage data from a \emph{generative model} of environments. As an example, consider the problem of manipulating mugs (Fig.~\ref{fig:anchor}); one could hand-craft a generative model that produces shapes that are similar to mugs (e.g., hollow cylinders; Fig. \ref{fig:anchor}) or potentially train a generative model over shapes using a dataset of different objects (e.g., bowls). 

% While relying entirely on real-world datasets can pose the risk of overfitting (since these datasets are generally small due to limited availability), training entirely on hand-engineered generative models would not lead to valid guarantees for real-world applications. In order to leverage the benefits associated with each kind of available data, we propose a method to combine the two sources to make guarantees on safety and performance in novel environments (environments not present in the robot's training dataset).\\

\begin{figure}[t]
\centering
\includegraphics[width=0.49\textwidth]{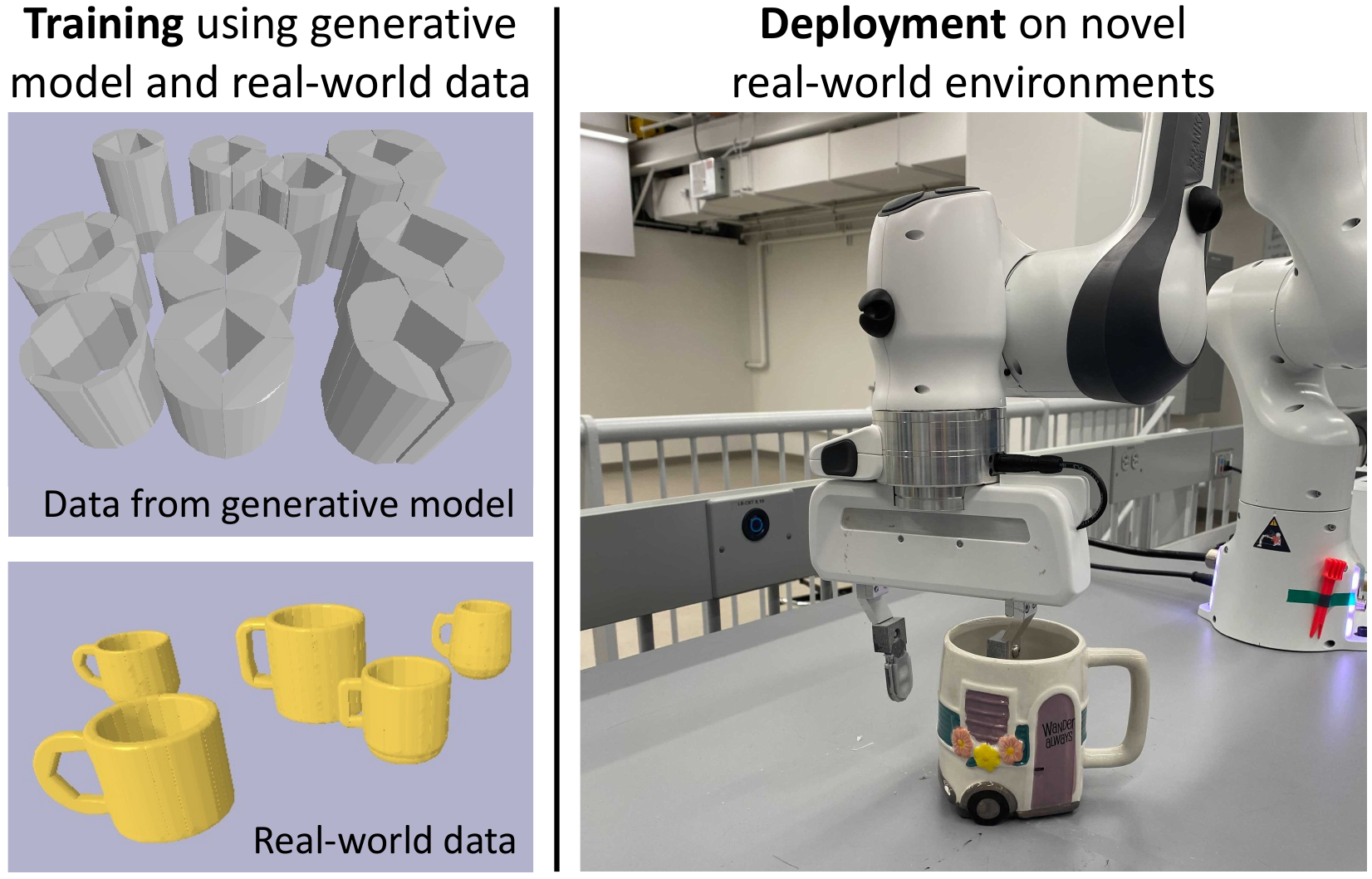}
\caption{A schematic of our overall approach. We provide a framework for providing \emph{generalization guarantees} on novel environments by combining a (potentially inaccurate) generative model of environments (e.g., a distribution that generates hollow cylinders) with a finite dataset of real-world environments (e.g., a datset of mugs). We validate our approach on hardware using the Franka Panda arm, and through multiple simulation experiments.
% train priors over policy space using objects from a generative model (in this case, cylinders). We then use real-world data (in this case, mugs) to find a posterior and obtain a PAC-Bayes bound to provide generalization guarantees. Finally, we test our policy distribution on hardware for the manipulation example.
}
\label{fig:anchor}
\vspace{-5mm}
\end{figure}

The two sources of data outlined above have complementary features: real-world data is scarce but representative, while data from a generative model is plentiful but potentially different from environments the robot will encounter when deployed. Thus, relying entirely on the small real-world dataset can pose the risk of overfitting, while relying entirely on the generative model may cause the robot to overfit to the specific features of this model and prevent generalization to real-world environments. How can we effectively combine these two sources in order to \emph{guarantee} that the robot will generalize to novel real-world environments? 

% \textit{Statement of contributions: } To our knowledge, the framework described in this paper is the first attempt to leverage a combination of a finite real-world dataset and a (potentially inaccurate) hand-engineered generative model in order to provide guarantees on safety and/ or performance in novel environments. We utilize PAC-Bayes generalization theory and provide a way to embed a generative model as a prior. The posterior is then trained using the real-world dataset. The policy space over which the posterior distribution is trained is parameterized by datasets drawn from the generative model. The resulting approach learns a policy with a guaranteed generalization bound. In order to demontrate the ability of our approach to provide strong generalization bounds, we provide two examples which use vision input: navigation of a UAV through a set of obstacles and grasping of mugs by a robotic arm. For both examples, we obtain PAC-Bayes bounds that guarantee successful completion of the task in 80-90 \% of the cases.  \\

\textit{Statement of contributions: } We provide a framework for providing \emph{formal guarantees on generalization} to novel environments for robotic systems with rich sensory inputs by leveraging a combination of \emph{finite real-world data} and a (potentially inaccurate) \emph{generative model} of environments. To our knowledge, the approach presented here is the first to leverage these two sources of data while providing generalization guarantees for robotic systems.
% We propose a framework for leveraging a combination of finite real-world data and a (potentially inaccurate) generative model in order to provide \emph{formal guarantees on generalization} to novel environments for robotic systems with rich sensory inputs. 
The key technical insight behind our approach is to utilize the generative model for specifying a \emph{prior} over control policies. In order to achieve this, we develop a technique for \emph{implicitly} parameterizing policies via datasets of environments. We then train a \emph{posterior distribution} over policies using the real-world dataset; this posterior is trained to minimize an \emph{upper bound} on the expected cost across novel environments derived via \emph{Probably Approximately Correct (PAC)-Bayes} generalization theory. Minimizing the PAC-Bayes bound allows us to automatically trade-off reliance on the real-world dataset and the generative model, while resulting in policies with a guaranteed bound on expected performance in novel environments. We demonstrate our approach on two examples which use vision inputs: (i) navigation of an unmanned aerial vehicle (UAV) through obstacle fields, and (ii) grasping of mugs by a robotic manipulator. For both examples, we obtain PAC-Bayes bounds that guarantee successful completion of the task in 80--95\% of novel environments. Comparisons with prior work demonstrate the ability of our approach to obtain stronger generalization guarantees by utilizing generative models. We also present hardware experiments for validating our bounds on the grasping task. 

\subsection{Related Work}

{\bf Domain randomization and data augmentation.} Domain randomization (DR) is a popular technique for improving the generalization of policies learned via reinforcement learning (RL). DR generates new training environments by randomizing specified dynamics and environmental parameters, e.g., object textures, friction properties, and lighting conditions \cite{tobin2017domain, peng2018sim, tan2018sim, akkaya2019solving, mehta2020active}, or generating new objects for manipulation by combining different shape primitives \cite{tobin2018domain}. Similarly, data augmentation techniques such as random cutout and cropping \cite{kostrikov2020image, laskin2020reinforcement} seek to improve generalization for vision-based RL tasks by performing transformations on the observation space. While these techniques have been empirically shown to improve generalization, they do not provide any guarantees on generalization (which is the focus of our work).

{\bf Generative modeling of environments.} Domain randomization techniques do not necessarily generate realistic environments for training. Consequently, another line of work seeks to address this challenge by generating environments with more realistic structure, e.g., via scene grammars and variational inference \cite{izatt2020generative, qi2018human, kar2019meta}, procedural generation \cite{cobbe2020leveraging}, or evolutionary algorithms \cite{morrison2020egad}. Adversarial techniques have also been developed for generating challenging environments \cite{wang2019adversarial, ren2021distributionally}. Prior work has also explored augmenting real-world training with large amounts of procedurally generated environments via domain adaptation techniques \cite{bousmalis2018using}, transfer learning \cite{kang2019generalization}, or fine-tuning \cite{kulhanek2021visual}. We highlight that none of the methods above provide guarantees on generalization to real-world environments. In this work, we provide a framework based on PAC-Bayes generalization theory in order to combine environments from a generative model with real-world environments and provide generalization guarantees for the resulting policies. Our work is thus complementary to the above techniques and could potentially leverage advances in generative modeling. 

{\bf Generalization theory.} Generalization theory provides a framework for learning hypotheses (in supervised learning) with guaranteed bounds on the true expected loss on new examples drawn from the underlying (but unknown) data-generating distribution, given only a finite number of training examples. Early frameworks include Vapnik-Chervonenkis (VC) theory \cite{Vapnik68} and Rademacher complexity \cite{Shalev14}. However, these methods often provide vacuous generalization bounds for high-dimensional hypothesis spaces (e.g., neural networks). Bounds based on PAC-Bayes generalization theory \cite{Shawe-Taylor97, McAllester99, Seeger02} have recently been shown to provide strong generalization guarantees for neural networks in a variety of supervised learning settings \cite{Dziugiate17, Langford03, Germain09, Bartlett17, Jiang20, Perez-Ortiz20}, and have been significantly extended and improved \cite{Catoni04, Catoni07, McAllester13, Rivasplata19, Thiemann17, Dziugaite18}. PAC-Bayes has also recently been extended to learn policies for robots with guarantees on generalization to novel environments \cite{Majumdar18, Veer20, Ren20, Majumdar21}. In this paper, we build on this work and provide a framework for leveraging generative models as a form of prior knowledge within PAC-Bayes. Comparisons with the approaches presented in \cite{Veer20, Ren20} demonstrate that this leads to stronger generalization guarantees and empirical performance (see Section \ref{sec:examples} for numerical results).

\section{Problem Formulation}\label{sec:prob-form}
{\bf Dynamics, environments, and sensing.} Consider a robotic system with discrete-time dynamics given by:
\begin{equation}
    x_{t+1} = f_E(x_t, u_t),
\end{equation}
where $x_t \in$ $\mathcal{X} \subseteq \mathbb{R}^{n_x}$ is the state of the robot at time-step $t$, $u_t \in \mathcal{U} \subseteq \mathbb{R}^{n_u}$ is the control input, and $E \in \mathcal{E}$ is the environment that the robot is operating in. The term ``environment" is used broadly to represent all external factors which influence the evolution of the state of the robot, e.g., an obstacle field that a UAV has to avoid, external disturbances such as wind gusts, or an object that a robotic manipulator is grasping. The dynamics of the robot may be nonlinear/hybrid. Let $\mathcal{O} \subseteq \mathbb{R}^{n_o}$ denote the space corresponding to the robot's sensor observations (e.g., the space of images for a camera).

% {\bf Sensing.} Let $\mathcal{O} \subseteq \mathbb{R}^{n_o}$ be the space corresponding to the robot's sensor observations (e.g., the space of images for a camera). Then g: $\mathcal{X}$ $\times$ $\mathcal{E}$ $\to$ $\mathcal{O}$ represents the robots extroreceptive sensor that produces observation $o \in \mathcal{O}$. Consider $\mathcal{L}$ as a space over motion primitives (time-varying proprioceptive controllers) which are executed in a \textit{receding horizon manner}. Our aim is to train policy $\pi$: $\mathcal{O}$ $\to$ $\mathcal{L}$. 

% In order to facilitate the training, we define a cost function C($\pi$, $E$). The environment $E$ is able to capture all external factors on which the cost of executing a policy $\pi$ might depend. For instance, $E$ has information about location and size of obstalces, and a collision on executing policy $\pi$ would represent a cost of 1 while no collision would represent a cost of 0. All the assumptions in \cite{Veer20} for the cost function hold in order to be able to use PAC-Bayes theory. \\

{\bf Policies and cost functions.} Let $\pi: \mathcal{O} \rightarrow \mathcal{U}$ be a policy that maps observations (or potentially a history of observations) to actions, and let $\Pi$ denote the space of policies  (e.g., neural networks with a certain architecture). The robot's task is specified via a cost function; we let $C_E(\pi)$ denote the cost incurred by policy $\pi$ when deployed in environment $E$ over a time horizon $T$. As an example in the context of UAV navigation, the cost function can assign 1 if the UAV collides with an obstacle, or 0 if it successfully reaches its goal. We assume that the cost is bounded, and without further loss of generality assume that $C_E(\pi) \in [0,1]$. Importantly, we make no further assumptions on the cost function (e.g., we \emph{do not} assume continuity or Lipschitzness). 

% Similar to previous work in this area \cite{majumdar2020pacbayes}, we assume an underlying \textit{unknown} distribution $\mathcal{D}$ over the space $\mathcal{E}$ of all environments from which the set $S := \{E_1, E_2, ..., E_N\}$ can be drawn i.i.d. for training. For the remainder of the paper, we refer to the environments in $S$ as real environments, since this is analogous to the \textit{real-world datasets} mentioned in the introduction. \\

{\bf Dataset of real-world environments.} We assume that there is an underlying distribution $\mathcal{D}$ from which real-world environments that the robot operates in are drawn (e.g., an underlying distribution over obstacle environments for UAV navigation, or objects for grasping). Importantly, we \emph{do not} assume that we have explicit knowledge of $\mathcal{D}$ or the space $\mathcal{E}$ of real-world environments. Instead, we assume access to a finite dataset $S := \{E_1, E_2, ..., E_N\}$ of $N$ real-world environments drawn independently from $\mathcal{D}$. 

{\bf Generative model.} In addition to the (potentially small) dataset of real-world environments, we assume access to a \emph{generative model} over environments. This generative model takes the form of a distribution $\mathcal{D}_\text{gen}$ over a space $\mathcal{E}_\text{gen}$ of environments. Importantly, $\mathcal{D}_\text{gen} \neq \mathcal{D}$ and $\mathcal{E}_\text{gen} \neq \mathcal{E}$ in general. Indeed, the space $\mathcal{E}_\text{gen}$ will typically be significantly simpler than the space $\mathcal{E}$ of real-world environments. For example, in the context of manipulation (Fig. \ref{fig:anchor}), $\mathcal{E}$ may correspond to the space of all mugs while $\mathcal{E}_\text{gen}$ may correspond to the space of hollow cylinders (described by a small number of geometric and physical parameters).

{\bf Goal.} Our goal is to learn a policy that \emph{provably generalizes} to novel real-world environments drawn from $\mathcal{D}$. In this paper, we will employ a slightly more general formulation where we choose a \emph{distribution} $P$ over policies (instead of choosing a single policy). This allows for the use of PAC-Bayes generalization theory. Our goal is then to tackle the following optimization problem:
\begin{equation}\label{eq:OPT}
    \min_{P \in \mathcal{P}} \ C_\mathcal{D}(P), \ \text{where} \ C_\mathcal{D}(P) := \mathop{\mathbb{E}}_{E \sim \mathcal{D}} \mathop{\mathbb{E}}_{\pi \sim P}[C(\pi; E)].
\end{equation}
The primary challenge in tackling this problem is that the distribution $\mathcal{D}$ is \emph{unknown} to us. Instead, we have access to a finite number of real-world environments and a (potentially inaccurate) generative model. In the next section, we describe how to leverage these two sources of data in order to learn a distribution $P$ over policies with a \emph{guaranteed bound} on the expected cost $C_\mathcal{D}(P)$, i.e., a provable guarantee on generalization to novel environments drawn from $\mathcal{D}$.

\section{Generalization Guarantees with \\ Generative Models}

In this section, we describe how to combine generative models with a finite amount of real data in order to produce strong generalization guarantees via PAC-Bayes theory. % introduce PAC-Bayes generalization bounds and then leverage generative models to provide strong PAC-Bayes bounds with a finite amount of real data.

\subsection{PAC-Bayes Control}
Our objective is to solve the optimization problem \eqref{eq:OPT}. However, the lack of an explicit characterization of $\mathcal{D}$ prohibits us from directly minimizing $C_\mathcal{D}(P)$. PAC-Bayes generalization bounds \cite{McAllester99} provide a high-confidence upper bound on $C_\mathcal{D}(P)$ in terms of the empirical cost on the training environments $S$ that are drawn from $\mathcal{D}$ and a regularizer. As both these terms can be computed, we minimize the PAC-Bayes upper bound in order to indirectly minimize $C_\mathcal{D}(P)$. Additionally, the PAC-Bayes bound serves as a certificate of generalization to novel environments drawn from $\mathcal{D}$.

Let $\Pi:=\{\pi_\theta~|~\theta\in\Theta\subseteq\mathbb{R}^{n_\theta}\}$ denote the space of policies parameterized by the vector $\theta$; as an example, $\theta$ could be the weights and biases of a neural network. For a ``posterior" policy distribution $P$ on $\Pi$ and a real-world dataset $S \coloneqq \{E_{1}, E_{2}, \cdots ,E_{N} \}$ of $N$ environments drawn i.i.d from $\mathcal{D}$, we define the \emph{empirical cost} as the expected cost across the environments in $S$:
\begin{align}
    C_{S}(P) \coloneqq \frac{1}{N} \sum_{E \in S} \mathop{\mathbb{E}}_{\theta \sim P} [C(\pi_\theta, E)].
    \label{eq:emp_cost}
\end{align}

Let $P_0$ be a ``prior" distribution over $\Pi$ which is specified before the training dataset $S$ is observed. The PAC-Bayes theorem below then provides an upper bound on the true expected cost $C_\mathcal{D}(P)$ which holds with high probability.

\begin{theorem}[adapted from \cite{Veer20}]\label{thm:pac-bayes}
For any $\delta\in(0,1)$ and posterior $P$, with probability at least $1-\delta$ over sampled environments $S\sim \mathcal{D}^N$, the following inequality holds:
\begin{align}
~C_{\mathcal{D}}(P) & \leq C_{PAC}(P,P_0) \nonumber \\
&  := \big(\sqrt{C_S(P) + R(P,P_0)} + \sqrt{R(P,P_0)}\big)^2 , \label{eq:quad-pac-bound}
\end{align}
where
% \footnote{The symbol $C_{QPAC}$ is used to identify that this bound follows from the ``quadratic" PAC-Bayes bound \cite{Rivasplata19}.} 
$R(P,P_0)$ is a regularization term defined as:
\begin{equation}\label{eq:R}
R(P,P_0):=\frac{\textrm{KL}(P||P_0) + \log\big(\frac{2\sqrt{N}}{\delta}\big) }{2N} \enspace.
\end{equation} 
\end{theorem}

It is challenging to specify good priors $P_0$ on the policy space $\Pi$ in general (e.g., specifying a prior on neural network weights); our previous approaches resorted to techniques such as data splitting \cite{Veer20} and imitation learning \cite{Ren20} to obtain priors. On the other hand, generative models offer an intuitive approach for embedding prior domain knowledge in learning \cite{Sunderhauf18,izatt2020generative,qi2018human,kar2019meta,cobbe2020leveraging}. Motivated by this, we will leverage generative models (based on inductive bias or other data) as priors for the PAC-Bayes theorem.

\subsection{Policy Parameterization With Datasets}
\label{subsec:policy-param}

The posterior $P$ and the prior $P_0$ distributions in Theorem~\ref{thm:pac-bayes} are on the space $\Pi$ of policies. Our key idea for leveraging generative models to provide generalization guarantees is to provide an approach for \emph{implicitly} parameterizing policies $\pi_\theta$ via synthetic datasets drawn from the generative model. This parameterization is then used in Theorem~\ref{thm:pac-bayes} such that the PAC-Bayes bound is specified in terms of the posterior $Q$ and the prior $Q_0$ on the space $\mathcal{E}_\text{gen}$ of synthetic environments. Let $\hat{S}$ be a synthetic (i.e., generated) dataset of cardinality $l$ and let $L:\Pi\times\mathcal{E}_\text{gen}^l\to [0,\infty)$ be a loss function; e.g., $L$ can be the average cost of deploying a policy $\pi_\theta$ in environments in $\hat{S}$. Then, let $A:\mathcal{E}_\text{gen}^l\to\Theta$ be an arbitrary \emph{deterministic algorithm} for (approximately) solving the optimization problem: 
\begin{align}\label{eq:parameterization-obj}
    \arg\inf_{\theta\in\Theta} L(\pi_\theta, \hat{S}) \enspace.
\end{align}
Any such algorithm then provides a way to parameterize policies $\pi_{A(\hat{S})}$ implicitly via datasets $\hat{S}$. 
% encodes the policy $\pi_{A(\hat{S}}$ in a dataset $\hat{S}$. 
We note that we do not impose any additional conditions on $A$ (e.g., $A$ need not solve \eqref{eq:parameterization-obj} to global/local optimality). Moreover, although we require $A$ to be deterministic, we can use stochastic optimization approaches --- such as stochastic gradient descent --- by fixing a random seed (this ensures deterministic outputs for a given input). 
% A fixed random seed would ensure that the same input always produces the same output resulting in $A$ being deterministic 
The algorithm $A$ gives rise to a push-forward measure for distributions from the synthetic environment space $\mathcal{E}_\text{gen}$ to the policy space $\Pi$. We overload the notation to express the push-forward distribution on the policy space as $A(Q)$.

% For a given dataset $\hat{S}$, we want to encode the policy $\pi_\theta$ which minimizes the 

% Towards that objective, we first define $h:\mathcal{E}_{\rm gen}^l\to \Pi$ which maps a synthetic dataset $\hat{S}:=\{E_{\mathrm{gen},1}, \cdots, E_{\mathrm{gen},l}\}$ of cardinality $l$ to a policy as follows:
% where $L$ is a loss function. This map will act as a bridge between the space of policies and the space of synthetic data. 

\subsection{PAC-Bayes Bounds With Generative Models}
\label{subsec:pac-bayes-gen}

In order to provide PAC-Bayes bounds using generative models, we encode the posterior $P$ and the prior $P_0$ on the policy space via posterior $Q$ and prior $Q_0$ generative models as follows: $P = A(Q)$, and $P_0 = A(Q_0)$.
% Further, with an abuse of notation, we express the true cost $C_\mathcal{D}$ and the empirical cost $C_S$ for policy distributions encoded by generative models as $C_\mathcal{D}(Q)$ and $C_S(Q)$ instead of $C_\mathcal{D}(A(Q))$ and $C_S(A(Q)$, respectively. 
We are now ready to present the PAC-Bayes bound with generative models.
\begin{theorem}\label{thm:pac-bayes-gen}
Let $A$ be a deterministic algorithm as defined above. For any $\delta\in(0,1)$ and posterior generative model $Q$ on $\mathcal{E}_\text{gen}$, with probability at least $1-\delta$ over sampled real-world environments $S\sim \mathcal{D}^N$, the following holds:
\begin{align}
~C_{\mathcal{D}}(A(Q)) & \leq C_{PAC}(Q,Q_0) \nonumber \\
&  := \big(\sqrt{C_S(A(Q)) + R(Q,Q_0)} + \sqrt{R(Q,Q_0)}\big)^2 , \label{eq:gen-quad-pac-bound}
\end{align}
where
\begin{align}\label{eq:data-gen-pac-bayes-emp}
    C_S(A(Q)) := \frac{1}{N} \sum_{E \in S} \mathop{\mathbb{E}}_{\hat{S} \sim Q} [C(\pi_{A(\hat{S})}, E)]
\end{align}
and $R(Q,Q_0)$ is the same as \eqref{eq:R}.
\end{theorem}
\begin{proof}
The proof follows by choosing $A(Q)$ as the posterior policy distribution $P$ and $A(Q_0)$ as the prior policy distribution $P_0$ in \eqref{eq:quad-pac-bound}, giving us the following bound:
\begin{align}
C_{\mathcal{D}}(A(Q)) & \leq \big(\sqrt{C_S(A(Q)) + R(A(Q),A(Q_0))} \\
 & \phantom{\leq} + \sqrt{R(A(Q),A(Q_0))}\big)^2 ,\label{eq:inter-1}
\end{align}

The empirical cost can be expressed as:
\begin{align}
    C_{S}(A(Q)) = \frac{1}{N} \sum_{E \in S} \mathop{\mathbb{E}}_{\theta\sim A(Q)} [C(\pi_\theta, E)].
\end{align}
Sampling $\theta$ from the push-forward measure $A(Q)$ is equivalent to sampling $\hat{S}$ from $Q$ and then computing $A(\hat{S})$. Therefore, the empirical cost can be expressed as \eqref{eq:data-gen-pac-bayes-emp}.
% \begin{align}\label{eq:data-gen-pac-bayes-emp}
%     C_{S}(h(Q)) = \frac{1}{N} \sum_{E \in S} \mathop{\mathbb{E}}_{\hat{S} \sim Q} [C(h(\hat{S}), E)].
% \end{align}

Using the data processing inequality \cite{Cover99} we have $KL(A(Q)||A(Q_0)) \leq KL(Q||Q_0)$, which further results in $R(A(Q),A(Q_0)) \leq R(Q,Q_0)$. Using this in \eqref{eq:inter-1} completes the proof.
\end{proof}
% Minimizing the PAC-Bayes generalization bound in Theorem~\ref{thm:pac-bayes-gen} gives us a posterior generative model $Q$ which captures the features of the real environments drawn from $\mathcal{D}$ that are relevant to the robot learning task, as encoded by the cost $C$. In particular, for synthetic data $\hat{S}$ drawn from $Q$, the policy $\pi_{A(\hat{S})}$ parameterized by $\hat{S}$ (which performs well on $\hat{S}$ as $A$ is solving \eqref{eq:parameterization-obj}) performs well on real environments drawn from the true distribution $\mathcal{D}$. Effectively, 
Minimizing $C_{PAC}$ provides us a policy distribution $A(Q)$ with a guaranteed bound on the expected cost $C_\mathcal{D}$ on novel environments, thereby tackling the optimization problem \eqref{eq:OPT}.

\section{Training}
\label{sec:training}

In this section, we present our training pipeline for combining a generative model with real-world data in order to provide strong generalization guarantees. First, we describe the algorithm $A$ used for parameterizing policies through datasets (Sec.~\ref{subsec:policy-param}). Then we provide the algorithm for minimizing the PAC-Bayes upper bound in Theorem~\ref{thm:pac-bayes-gen}.

% In this section we describe how to obtain an optimal distribution $P$ over the policy space $\Tilde{\Pi}$. First, we describe the evolutionary strategies framework used to execute the function h as given in equation 4, which helps us define $\Tilde{\Pi}$. Then we introduce the relative entropy programming method used to find the posterior $P$.

\subsection{Policy Parameterization With Datasets}

As discussed in Sec. \ref{subsec:policy-param}, we require a deterministic algorithm $A$ (that attempts to minimize a loss $L$) in order to implicitly parameterize policies $\pi_{A(\hat{S})}$ via datasets $\hat{S}$. For the results in this paper, we use $L$ as the average cost of deploying a policy $\pi_\theta$ in environments contained in $\hat{S}$:
\begin{align}
    L(\pi_\theta,\hat{S}):= \frac{1}{l} \sum_{E_\text{gen} \in \hat{S}} C(\pi_\theta, E_\text{gen}).
\end{align}

To minimize $L$, we choose the algorithm $A$ to be Evolutionary Strategies (ES) \cite{Wierstra14} with an a priori fixed random seed; fixing the random seed ensures that the algorithm is deterministic. ES belongs to a family of black-box optimizers which train a distribution on the policy space. The choice of ES is driven by our use of black-box simulators through which the gradient of the loss cannot be backpropagated (e.g., due to the loss being non-differentiable or due to the dynamics of the robot being hybrid). Additionally, ES permits a high degree of parallelization, thereby allowing us to effectively exploit clouding computing resources. In the interest of space, further details on our implementation of ES are not provided here and can be found in \cite[Sec.~4.1]{Veer20}.

\subsection{Training a PAC-Bayes Generative Model}

We assume availability of a generative model expressed by a distribution $\mathcal{D}_\text{gen}$ on $\mathcal{E}_\text{gen}$ (ref. Sec. \ref{sec:prob-form}); this model could be hand-specified based on prior knowledge or constructed using other data. Leveraging $\mathcal{D}_\text{gen}$ we first construct a prior generative model and then train a posterior generative model by minimizing the PAC-Bayes bound in Theorem~\ref{thm:pac-bayes-gen}.

As has been shown in \cite{Veer20} and \cite{Majumdar21}, PAC-Bayes minimization takes the form of an \emph{efficiently-solvable convex program} for discrete probability distributions. To exploit this convex formulation (which allows one to optimize the PAC-Bayes bound in a computationally efficient manner), we construct a prior generative model $q_0$ which approximates $\mathcal{D}_\text{gen}$ as a discrete probability distribution as follows:\\
\emph{Let $\mathcal{D}_\text{gen}$ be a generative model which takes the form of a distribution on the synthetic environment space $\mathcal{E}_\text{gen}$, as discussed in Sec.~\ref{sec:prob-form}. Sample $m$ datasets of cardinality $l$ each from $\mathcal{D}_\text{gen}$ to construct the set of datasets $\tilde{S}:=\{\hat{S}_1,\cdots,\hat{S}_m~|~\hat{S}_i\sim \mathcal{D}_\text{gen}^l\}$. The prior generative model $q_0$ is then defined as the uniform distribution on $\hat{S}$.}

To train a posterior generative model $q$ (which is a discrete probability distribution on the set $\tilde{S}$ of synthetic datasets), we minimize the PAC-Bayes upper bound in Theorem~\ref{thm:pac-bayes-gen}. To transform this minimization into a convex program, we first compute a cost vector $C\in\mathbb{R}^m$. Each entry $C_i$ of this vector corresponds to the expected cost of deploying the policy $\pi_{A(\hat{S}_i)}$, parameterized by the synthetic dataset $\hat{S}_i$, in the real-world training dataset $S$. Therefore, the empirical cost $C_S(A(Q))$ can be expressed as $Cq$ (which is linear in the generative model posterior $q$). Leveraging this, we can express the PAC-Bayes bound minimization as follows:
\begin{align}
\min_{q\in\mathbb{R}^m} \quad & \big(\sqrt{Cq + R(q,q_0)} + \sqrt{R(q,q_0)}\big)^2 \label{eq:REP}\\
\textrm{s.t.} \quad & \sum_{i=1}^m q_i = 1, 0\leq q_i \leq 1. \nonumber
\end{align}
Using the epigraph trick, as detailed in \cite{Veer20}, \eqref{eq:REP} can be further transformed to a convex program. In the interest of space, we direct the reader to \cite[Sec.~4.2]{Veer20} for complete details of the algorithm to solve \eqref{eq:REP}. We provide a sketch of our entire training pipeline in Alg.~\ref{alg:train-pipeline}. 

\begin{algorithm}[h]
\caption{Training Pipeline \label{alg:train-pipeline}}
\small
\begin{algorithmic}[1]
	\State \textbf{Input:} Generative model: $\mathcal{D}_\text{gen}$; real-world dataset: $S\sim\mathcal{D}^N$
	\State \textbf{Input:} Number of synthetic datasets: $m$ 
	\State \textbf{Input:} Cardinality of each synthetic dataset: $l$ 
	\State \textbf{Input:} Deterministic algorithm for \eqref{eq:parameterization-obj}: $A$
	\State Sample $\hat{S}_1,\cdots, \hat{S}_m\sim\mathcal{D}_\text{gen}^l$
	\State $q_0 \gets [1/m, \cdots, 1/m] $
	\State $q\gets \text{PAC-Bayes}(S,A,q_0,\{\hat{S}_i\}_{i=1}^m)$ by solving \eqref{eq:REP} \\
	\Return $q$
	\end{algorithmic}
\normalsize
\end{algorithm}

% In order to be able to use \textit{convex optimization} methods, we restricted the policy space to $\Tilde{\Pi}$. We can now express the problem of training a distribution over $\Tilde{\Pi}$ as a relative entropy programming problem. Consider the prior $p_0$ and posterior $p$, both from $\mathcal{\Tilde{P}}$, over $\Tilde{\Pi}$. If the vector $C$ $\in$ $\mathcal{R}^m$, each element of which consists of the cost of running a policy from the finite set on the given environment. The PAC-Bayes expression from equation 7 then becomes: 
% \begin{equation}
%     \min_{p \in \mathcal{R}^m}\sqrt{Cp + R(p, p_0)} + \sqrt{R(p, p_0}
% \end{equation}

% such that $\sum_{i = 1}^m p_i$ = 1, $p_i$ $\in$ (0, 1). We can then set up the REP problem as done in \cite{Veer20}.

\section{Examples}
\label{sec:examples}

We demonstrate the ability of our framework to provide strong generalization guarantees for two robotic systems with nonlinear/hybrid dynamics and rich sensory inputs: a drone navigating obstacle fields using onboard vision, and a manipulator grasping mugs using an external depth camera. % For each task, our approach provides: (i) strong theoretical generalization bounds, and (ii) empirical test results within the theoretical bounds. 
All training is conducted on a \texttt{Lambda Blade} server with \texttt{2x Intel Xeon Gold 5220R} (96 threads), 760 GB of RAM, and 8 \texttt{NVIDIA GeForce RTX 2080}, each with 12 GB memory. We compare our bounds against those in previous works with similar examples. 

\subsection{Vision-based obstacle avoidance with a drone}

% \begin{figure}[t]
%     \centering
%     \includegraphics[width = 0.4\textwidth]{figures/quadrotor-navigation.png}
%     \caption{Scenario for vision-based planning in which a UAV has to navigate across an obstacle field with no prior knowledge of the environment.}
%     \label{fig:UAV}
% \end{figure}

%fig
\begin{figure}[t]
\centering
\subfigure[]
{
\includegraphics[trim={14cm 3cm 14cm 0cm},clip,width=0.23\textwidth]{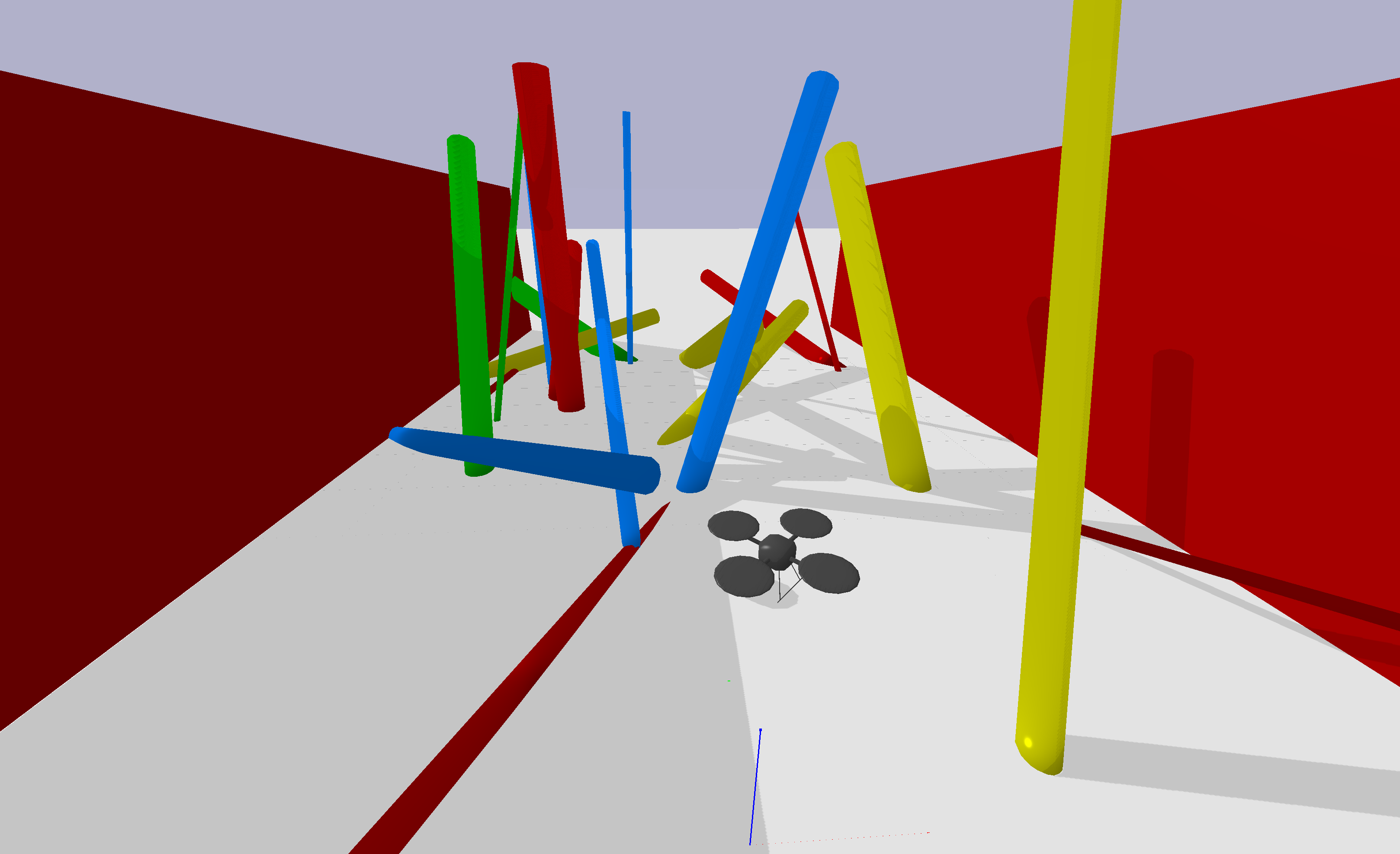}
\label{fig:UAV-env}
}
% \centering
\hspace{-3mm}
\subfigure[]
{
\includegraphics[width=0.23\textwidth]{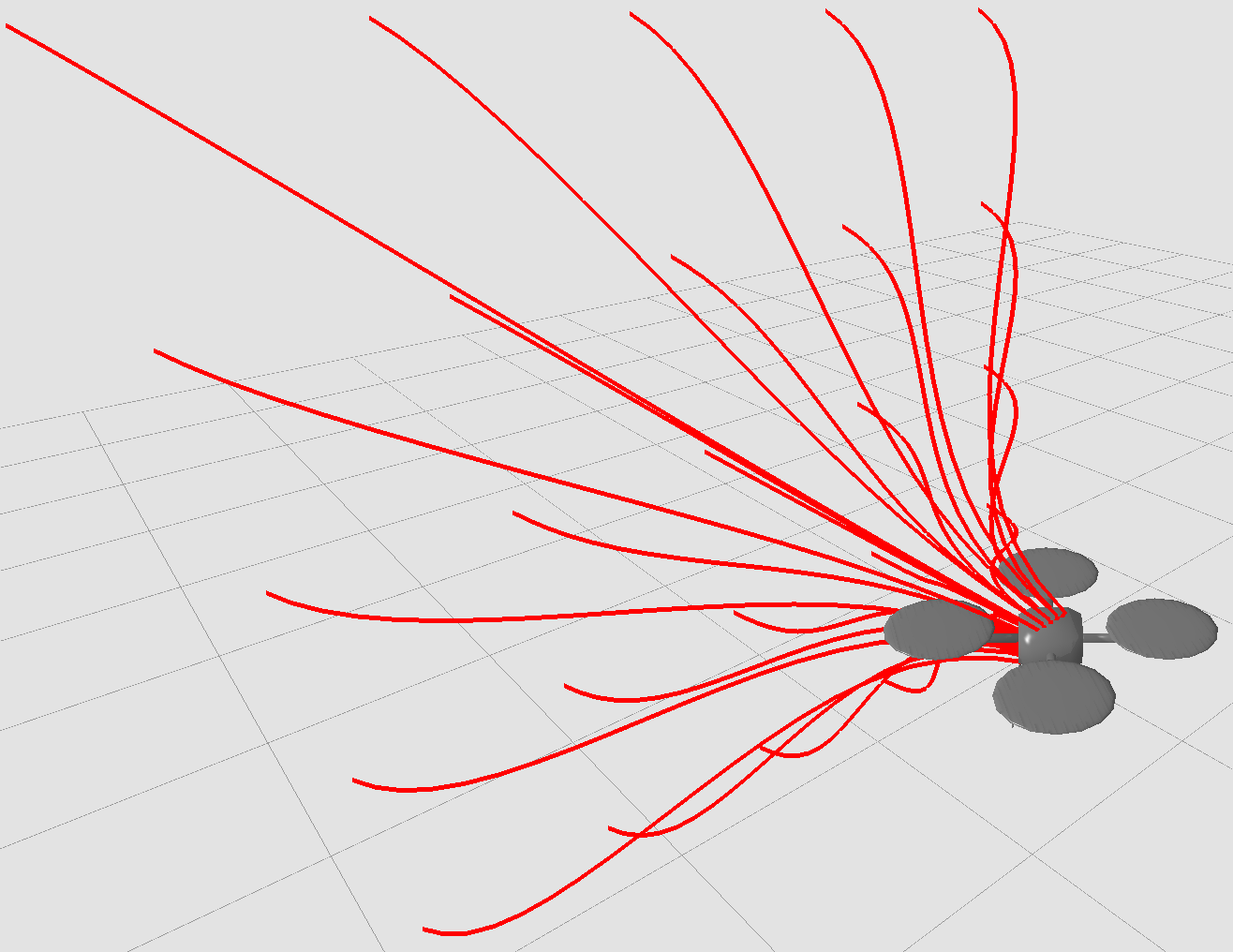}
\label{fig:UAV-prim}
}
\vskip -10pt
\caption{Vision-based navigation with a UAV. \textbf{(a)} Environment with randomly generated obstacles. \textbf{(b)} Primitive library for the UAV. \label{fig:UAV}
}
\vspace{-5mm}
\end{figure}
%fig

\textbf{Overview.} In this example, we train a quadrotor equipped with an onboard depth camera to navigate across obstacle fields. The obstacle course is a tunnel populated by cylindrical obstacles as shown in Fig. \ref{fig:UAV-env}. 
The dynamics and sensor are simulated using \texttt{PyBullet} \cite{pybullet}. % The robot learns to map depth images to a motion primitive output using the policy that we train. 

\textbf{Environments.} The distribution $\mathcal{D}$ over environments samples the radii, locations, and orientations of 23 obstacles in order to generate an environment; the radii are drawn from a uniform distribution over [5cm, 30cm], the locations of the center of the cylinders are drawn from [-5m, 5m]$\times$[0m, 14m], and the orientations are quaternion vectors drawn from a normal distribution. % In order to ensure that generated obstacle environments are traversable in principle, we implement a form of rejection sampling by randomly choosing a motion primitive from our library
% Both distributions on the space $\mathcal{E}$, $\mathcal{D_{\rm gen}}$ and $\mathcal{D}$, sample the radii of the obstacles from a uniform distribution over [5cm, 30cm] and the location of the center of the cylinder from [-5m, 5m]$\times$[0m, 14m]. The orientation is taken as a quaternion drawn from a normal distribution. The number of obstacles for $\mathcal{D}$ remains at 23, whereas for $\mathcal{D_{\rm gen}}$ we vary them from 10 to 30 in steps of 5. While spawning each environment, we implement a form of rejection sampling. We randomly pick one motion primitive from our library and execute it over the spawned environment. If it collides with an obstacle, that obstacle is used before the environment is added to the dataset. This is to make sure that at least a part of the environment is solvable (that is the drone can take at least the first few steps).

\textbf{Generative model.} The generative model $\mathcal{D_{\rm gen}}$ samples radii, locations, and orientations of obstacles from the same distributions as $\mathcal{D}$. However, the number of obstacles in each environment drawn from $\mathcal{D_{\rm gen}}$ is different from the number of obstacles in environments drawn from $\mathcal{D}$. In our experiments, we will study the effects of degrading the quality of the generative model by varying this parameter. 

\textbf{Motion primitives and planning policy.} We pre-compute a library of 25 motion primitives (Fig. \ref{fig:UAV-prim}), each of which is generated by connecting the initial position of the robot to a desired final position by a smooth sigmoidal trajectory. The robot's policy takes a $50 \times 50$ depth image from the onboard camera as input and selects a motion primitive to execute. This policy is applied in a receding-horizon manner (i.e., the robot selects a primitive, executes it, selects another primitive, etc.). The policy is parameterized using a deep neural network ($\sim$14K parameters) and is based on the policy architecture presented in \cite[Sec.~5.1]{Veer20}.

\textbf{Training.} We choose the cost $1 - \frac{k}{K}$ where $k$ is the number of motion primitives successfully executed before colliding with an obstacle and $K$ is the total possible primitive executions; in our example $K = 12$. 
% We choose a total of T = 12 time-steps, i.e., there are 12 instances when the depth map input is taken and motion primitve is executed. The cost assigned to each environment is 1 - $\frac{t}{T}$, where t is the time after which the robot collides with an obstacle. 
We train policies via the pipeline described in Section \ref{sec:training}. We choose $m=50$ datasets in $\tilde{S}$, and each dataset $\hat{S}_i \in \tilde{S}$ has cardinality $l=50$. With 6 GPUs and 48 CPUs, it takes $\sim$ 6-8 hours to train the priors and $\sim$ 200-1000 seconds to train the posterior (depending on the number of real environments used). 

% The priors are trained using the methods described in section IIIB. We train 6 sets of priors for comparison, with number of obstacles as 10, 15, 20, 23, 25, and 30. For each case, the cardinality of $\Tilde{\Pi}$ is 50, i.e. we draw 50 datasets from $\Tilde{D_{\rm gen}}$ and find an optimal policy for each dataset to define $\Tilde{\Pi}$. The size of each dataset drawn is 50, i.e. A = 50. The training is conducted on 

\textbf{Results.} We consider different generative models $\mathcal{D}_\text{gen}$ by varying the number of obstacles $N_\text{O}$ sampled in any generated environment; we vary this parameter in the set $\{10, 15, 20, 23, 25, 30\}$.
Generalization guarantees are obtained using each variation of the generative model. We set $\delta = 0.01$ to have bounds that hold with probability 0.99.

\begin{figure}
    \centering
    \includegraphics[scale = 0.52]{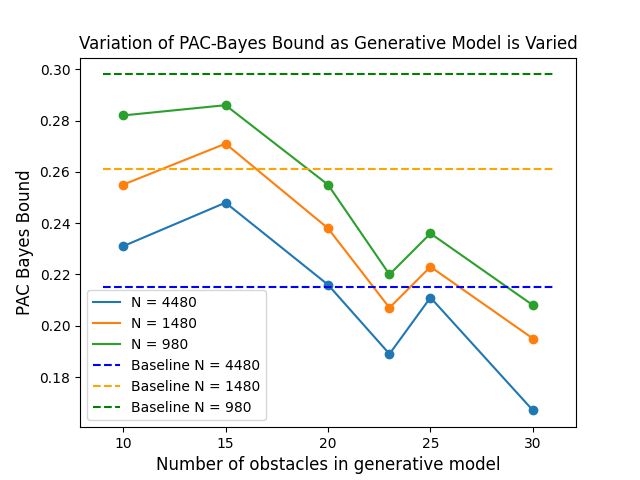}
    \caption{PAC-Bayes bounds for different choices of the generative model (obtained by varying the number $N_\text{O}$ of obstacles sampled in each environment). Bounds generally become stronger as we increase $N_\text{O}$. Comparisons with \cite{Veer20} (dotted lines) demonstrate the benefits of our approach, particularly for smaller values of $N$ (the number of available real-world environments).}
    \label{fig:UAV_Comparison}
    \vspace{-5mm}
\end{figure}

\begin{table}[h]
  \centering
  \renewcommand{\arraystretch}{1.2}
  \begin{adjustbox}{width=1\columnwidth,center}
  \begin{tabular}{|c|c|p{1.5cm}|c|p{1.5cm}|}
    \hline
    \multirow{2}{0.7 cm}{\textbf{Envs (N)}} & \multicolumn{2}{c|}{\textbf{Using generative model (ours)}} & \multicolumn{2}{c|}{\textbf{Approach from \cite{Veer20}}}\\
    % \hline
    % \textbf{Inactive Modes} & \textbf{Description}\\
    \cline{2-5}
    & \textbf{PAC Bound} & \textbf{True Cost (Estimate)} & \textbf{PAC Bound} & \textbf{True Cost (Estimate)}\\
    %\hhline{~--}
    \hline
    980 & 20.92 \% & 13.81 \% & 29.82 \% & 19.7 \% \\ \hline
    1480 & 19.60 \% & 13.81 \% & 26.02 \% & 18.34 \%  \\ \hline
    4480 & 16.76 \% & 13.86 \% & 21.52 \% & 18.43 \%  \\ \hline
  \end{tabular}
  \end{adjustbox}
  \caption{\footnotesize{Comparison of PAC-Bayes bounds with true expected cost on novel environments (estimated via exhaustive sampling). The framework presented here provides both stronger guarantees and empirical performance on novel environments as compared to \cite{Veer20}.}}
  \label{tab:table_1}
\end{table}

Figure \ref{fig:UAV_Comparison} plots the PAC-Bayes bounds on the expected cost for different choices of $N_\text{O}$. For example, when $N_\text{O}=30$ and $N=4480$, the PAC-Bayes bound using our approach is $0.1676$. Thus, we can guarantee that on average the quadrotor will successfully navigate through at least 83.24\% ($100\% - 16.76\%$) of novel real-world environments. 
We also compare these bounds with those provided by the method in \cite{Veer20} (plotted with dotted lines), which splits a given dataset of $N$ real-world environments into two portions; the first portion is used to train a prior over policies and the second portion is used to obtain a posterior distribution over policies by minimizing the PAC-Bayes bound in Theorem \ref{thm:pac-bayes}. We provide our approach with the same number of real-world environments (i.e., $N$) as used in \cite{Veer20} in order to ensure a fair comparison. For each $N$, the bounds generally become stronger as we increase the number of obstacles $N_\text{O}$ sampled by the generative model. 

Figure \ref{fig:UAV_Comparison} demonstrates that the approach presented here is able to produce stronger bounds than the ones provided by \cite{Veer20}, with significant differences when $N_\text{O} = 30$. Interestingly, the benefits of our approach become more apparent when the number $N$ of available real-world environments is small. For example, when $N=980$, the bounds provided by our approach are stronger for all choices of $N_\text{O}$. When $N$ is small, the prior information provided by the generative model becomes important (as one would intuitively expect). 

 % For easy interpretation of the data, consider the second last row with $N = 1480$: the quadrotor is guaranteed to traverse $77.7 \%$ of the previously unseen obstacle courses when trained using a generative model, whereas it is guaranteed to traverse $73.9 \%$ of previously unseen courses when priors are trained using real environments. 

Table \ref{tab:table_1} compares the theoretical generalization bounds obtained for the case when $N_\text{O}=30$ with the true expected cost on novel environments (estimated via exhaustive sampling of novel environments). Results are presented for different numbers $N$ of real-world environments for both our method and the one from \cite{Veer20}. As the table illustrates, our approach results in significantly improved performance on novel environments for all values of $N$.
% The table provides results of the empirical performance of the policy distribution trained using our method. \ani{Need to complete this.}

\vspace{-2mm}

\subsection{Grasping a diverse set of mugs}

\textbf{Overview.} This example aims to train a Franka Panda arm to grasp and lift a mug (Fig. \ref{fig:anchor}). The arm has an overhead camera which provides a 128 $\times$ 128 depth image. The simulation environment for this system is implemented using \texttt{PyBullet} \cite{pybullet}, and we also present hardware results on the Franka arm shown in Fig. \ref{fig:anchor} (right). 

% This image is used to compute an open loop action, which is then implemented as desired positions and orientations of the arm. 

\textbf{Environments.} The real-world environments used for training are drawn from a set of mugs with diverse shapes and sizes collected from the ShapeNet dataset \cite{chang2015shapenet}. The initial x-y position of these mugs is sampled from the uniform distribution over $[0.45~\text{cm},~0.55~\text{cm}] \times [-0.05~\text{cm},~0.05~\text{cm}]$, and yaw orientations are sampled from the uniform distribution over $[-\pi~\text{rad},\pi~\text{rad}]$. All mugs are placed upright. 
% All other cases, including the case in which the gripper palm touches the mug, are assigned a cost of 1.

\textbf{Generative model. }The generative model $\mathcal{D_{\rm gen}}$ comprises of hollow cylinders which are generated using \texttt{trimesh} \cite{trimesh}. The inner radii, outer radii, and height of the cylinders are sampled from uniform distributions. The ratio of the maximum possible outer radius to inner radius is 2, and the height ranges from twice the maximum inner radius to twice the maximum outer radius. The initial location and yaw are sampled from the same distributions as $\mathcal{D}$. 

\textbf{Policy.} The robot's policy is parameterized using a deep neural network (DNN) which takes a depth map of an object and a latent state $z\in\mathbb{R}^{10}$ sampled from a Gaussian distribution as input and outputs a grasp location and orientation. We keep the weights of the DNN fixed and update the distribution on the latent space. Effectively, the latent space acts as the space of policy parameters $\Theta$ and the Gaussian distribution on it is the policy distribution $P$; further details of the policy's architecture can be found in \cite{Ren20}. 

\textbf{Training.} If the arm is able to grasp and lift a mug by 10 cm, we consider the rollout to be successful and assign a cost of 0, otherwise we assign a cost of 1. We follow the pipeline in Alg.~\ref{alg:train-pipeline} for training. We choose m = 50 datasets in $\Tilde{S}$, with each dataset $\Tilde{S}_i$ $\in$ $\Tilde{S}$ having cardinality l = 50. With 80 CPUs, the priors train in $\sim$ 3 hours, and the posterior takes $\sim$ 900 seconds. 

\textbf{Simulation results.} We obtain theoretical generalization guarantees using the generative model described above and compare it with the theoretical guarantees obtained in \cite{Ren20}. We use the same set of 500 mugs from ShapeNet used by \cite{Ren20} as our real dataset in order to train the posterior and obtain the PAC-Bayes bound. Our resulting PAC-Bayes bound (with $\delta = 0.99$) is 0.054. Thus, our policy is guaranteed to have a success rate of at least 94.6 \%, which is higher than the 93 \% guaranteed success rate in \cite{Ren20} (despite using the same real-world dataset of mugs for training). 

\textbf{Hardware results.} The posterior policy distribution trained in simulation is deployed on the hardware setup shown in Fig. \ref{fig:anchor} without additional training (i.e., zero-shot sim-to-real transfer). 10 mugs with diverse shapes are used (Fig.~\ref{fig:mugs}). Among three sets of experiments with different seeds (for sampling the latent $z$), the success rates are 100\% (10/10), 100\% (10/10), and 90\% (9/10). The overall success rate is 96.67\% (29/30) and thus validates the PAC-Bayes bound of $94.6\%$ trained in simulation.

%fig
\begin{figure}[t]
\centering
\includegraphics[trim={0cm 18cm 0cm 18cm},clip,width=0.45\textwidth]{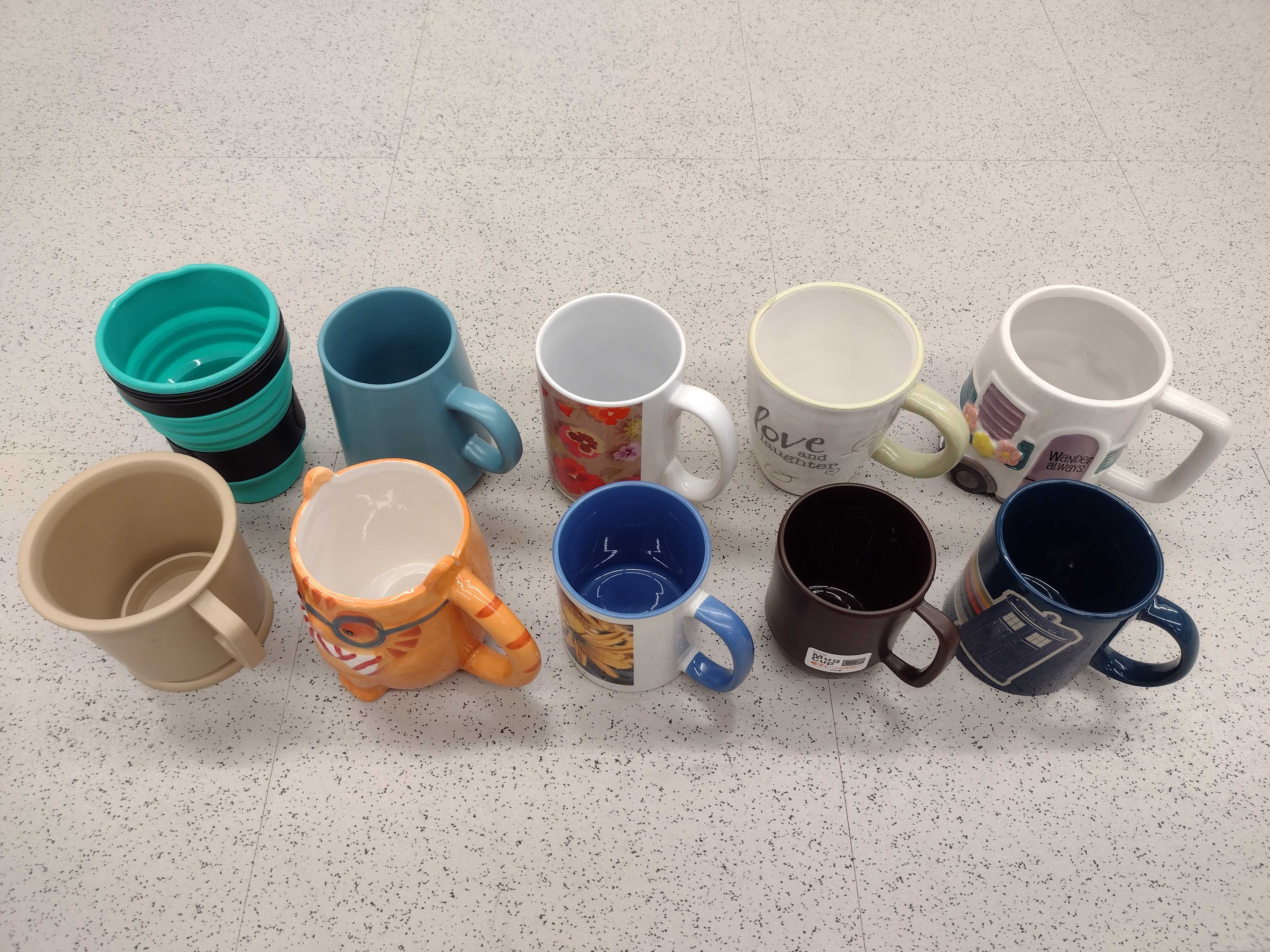}
% \vskip -10pt
\caption{Mugs used for hardware validation of the grasping policy. \label{fig:mugs}
}
\vspace{-7mm}
\end{figure}
%fig

\section{Conclusions and Future Work}

We have presented an approach for learning policies for robotic systems with \emph{guarantees on generalization} to novel environments by leveraging a finite dataset of real-world environments in combination with a (potentially inaccurate) generative model of environments. The key idea behind our approach is to use the generative model in order to implicitly specify a prior over policies, which is then updated using the real-world environments by optimizing generalization bounds derived via PAC-Bayes theory. Our simulation and hardware results demonstrate the ability of our approach to provide strong generalization guarantees for systems with nonlinear/hybrid dynamics and rich sensing modalities, and obtain stronger guarantees and empirical performance than prior methods that do not leverage generative models. 

Exciting directions for future work include (i) obtaining stronger guarantees by going beyond the hand-crafted generative models used here and using state-of-the-art techniques for generative modeling, (ii) directly optimizing a posterior generative model $Q$ in Theorem \ref{thm:pac-bayes-gen} (without performing the finite sampling described in Section \ref{sec:training}), and (iii) implementing the UAV navigation example on a hardware platform. 

% exploring other reinforcement learning algorithms in order to implement the algorithm $A$ that maps a dataset $\hat{S}$ to a policy (

% {\bf Future work.} This work gives rise to several exciting avenues for future work. 


%%%%%%%%%%%%%%%%%%%%%%%%%%%%%%%%%%%%%%%%%%%%%%%%%%%%%%%%%%%%%%%%%%%%%%%%%%%%%%%%

%%%%%%%%%%%%%%%%%%%%%%%%%%%%%%%%%%%%%%%%%%%%%%%%%%%%%%%%%%%%%%%%%%%%%%%%%%%%%%%%



% \section*{ACKNOWLEDGMENT}

% The preferred spelling of the word ÒacknowledgmentÓ in America is without an ÒeÓ after the ÒgÓ. Avoid the stilted expression, ÒOne of us (R. B. G.) thanks . . .Ó  Instead, try ÒR. B. G. thanksÓ. Put sponsor acknowledgments in the unnumbered footnote on the first page.

%%%%%%%%%%%%%%%%%%%%%%%%%%%%%%%%%%%%%%%%%%%%%%%%%%%%%%%%%%%%%%%%%%%%%%%%%%%%%%%%

% References are important to the reader; therefore, each citation must be complete and correct. If at all possible, references should be commonly available publications.

\bibliographystyle{IEEEtran} % We choose the &quot;plain&quot; reference style
\bibliography{refs} % Entries are in the &quot;refs.bib&quot; file</code></pre>

% \printbibliography

\end{document}